\documentclass[sigconf,screen,nonacm]{acmart}

\renewcommand\footnotetextcopyrightpermission[1]{}

\AtBeginDocument{%
	}

\usepackage{graphicx}
\usepackage{subcaption}

\usepackage{url}
\usepackage[disable]{todonotes}
\usepackage{listings}
\usepackage{amsmath}
\usepackage{tikz}
\usepackage{stmaryrd}
\usetikzlibrary{arrows.meta,positioning,fit,shapes.misc}

\usepackage{xcolor}
\usepackage{pgfplots}
\usepackage{pgfplotstable}
\usepackage{siunitx}
\usepackage{booktabs}
\usepackage{multirow}
\usepackage{calc}
\usepackage{enumitem}
\usepackage{algorithmic}

\pgfplotsset{compat=1.18}
\begin{document}
	
	\title{Overlaying Governance: A Compositional Authorization Framework for Delegation and Scope in Agentic AI}
	
	
	\author{Amjad Ibrahim}
	\affiliation{%
		\institution{Huawei Heisenberg Research Center}
\country{Germany}
	}
	\email{amjad.ibrahim@huawei.com}
	
	\author{Yong Li}
	\affiliation{%
		\institution{Huawei Heisenberg Research Center}
		\country{Germany}
	}
	\email{yong.li1@huawei.com}

	\renewcommand{\shortauthors}{First Author et al.}
	
			\begin{abstract}
	As AI systems evolve from passive models into autonomous active agents capable of initiating actions, collaborating, and delegating tasks, the traditional boundaries of software systems blur. Traditional authorization and delegation frameworks—built around fixed principals, explicit requests, and static scopes—are insufficient to govern agentic systems.
	Agentic AI demands richer authorization semantics: agents must inherit and delegate permissions, act under time-limited authority, and coordinate through shared protocols. Existing Identity and Access Management (IAM) systems fail to fully capture this notion of agency, lacking mechanisms for recursive delegation, contextual boundaries, and dynamic scoping as executable governance primitives.
	Unlike access delegation standards such as OAuth 2.0, we treat delegation as a contractual term rather than merely a static token-based consent credential. This paper proposes a compositional governance framework that introduces primitives indispensable for agentic AI. We define types of delegation and their permissions and accountability implications, and we introduce a notion of resource scope attenuation to bound agentic access envelopes. These concepts are expressed as general relational definitions that can be composed into existing authorization domains (e.g., financial systems).
	To operationalize this composition, we define a compositional operator that overlays new agentic semantics, such as recursive delegation chains, onto existing relational policies without rewriting them. We substantiate this framework through formal proofs and empirical evaluation, showing that it provides a formal yet practical foundation for accountable authorization in agentic AI systems.
\end{abstract}
	
	
	\maketitle
	
	
	\section{Introduction}\label{sec:intro}

The introduction of large language model (LLM)-based chat tools such as ChatGPT is already reshaping how we work~\cite{cazzaniga2024gen}, study, or conduct research~\cite{holmes2023guidance}. Each week, users exchange over 18 billion messages with ChatGPT, with a user-base accounting for nearly 10\% of the global adult population~\cite{gpt1}. While information retrieval remains the dominant use case, a new generation of AI systems is emerging that help users \emph{act} rather than merely \emph{ask}—so-called \emph{Agentic AI}. Agentic AI systems are software components that incorporate language models and can autonomously plan and execute actions based on user input and contextual awareness~\cite{openid,openAi,owasp}.\footnote{We use the terms \emph{Agentic AI} and \emph{AI agents} interchangeably.} 

Agentic AI is finding applications across domains such as finance (e.g., identity verification) and healthcare (e.g., patient monitoring~\cite{KARUNANAYAKE202573}). Since these agents tackle a wide range of tasks, \emph{collaboration and interoperability} become essential. To support this, industry initiatives such as the Model Context Protocol (MCP) define how agents integrate with tools, systems, and data sources, while the Agent2Agent (A2A) protocol specifies how agents from different vendors can search, and communicate with each other~\cite{agenticProto}. As a result, the traditional boundaries of software systems, once used to define trust domains and access control assumptions, no longer hold. Governance of such autonomous, interconnected agents using traditional methods has become impractical~\cite{Huang}.

From an \textit{authorization} perspective—the process of determining whether a principal may perform an action on a resource—Agentic AI introduces new challenges~\cite{openid,Huang,Syros}. Agents \textit{act on behalf} of users to achieve tasks, in the process they coordinate with other agents, and recursively delegate sub-tasks. Consequently, the active principal in an action may be a human, an agent delegated from a user, or an agent spawned by another agent. Even if we can distinguish between these actors, we must still determine their permissions: should agents inherit all user permissions, or should there be distinct \emph{types} of delegation? How can we govern and account for recursive delegations and contextual constraints without redesigning every existing authorization model individually?

Access delegation standards such as OAuth~2.0 provide limited delegation through access tokens, allowing applications to act on behalf of users within pre-defined \emph{scopes} without sharing credentials~\cite{oauth}. Scopes are typically a subset of user capabilities. While effective for static service workflows, OAuth is ill-suited to dynamic, recursive delegation chains~\cite{Huang,openid}. To support recursion, we can issue multiple nested tokens. Regardless of their operational overhead, the content of these tokens is fixed. Agents can generate novel actions that were not pre-enumerated in a scope~\cite{authenicatedDEL}. Thus, traditional token consent mechanisms cannot express the dynamic, and recursive authorization relations required in agentic ecosystems.  

We treat \emph{delegation} as a first-class governance primitive in Agentic AI authorization. Inspired by the human (legal) notion of delegation as a \emph{contractual transfer of duties}~\cite{LII_Wex_delegate_misc}, we define an \emph{agentic delegation} as a runtime predicate (term) that carries constraints, e.g., ``expires in 10 seconds'' or ``valid only on secure hardware.'' These predicates form a chain that represents the delegation state and enables dynamic contextual evaluation of authority.

Similarly, we view \emph{scope} as a set of contextual boundaries or \emph{envelopes} that constrain what delegation covers. Delegations along a chain must progressively narrow in scope, a property often referred to as \emph{attenuation}~\cite{openid}. Scopes determine the permissible range of resources, e.g., ``allow agent to edit proposals but not budgets.'' With the dynamic nature of these agents, specifying the allowed (or denied) range of resources is more practical~\cite{authenicatedDEL}. Together, delegation and scoping form  contractual \textit{primitives} for governing how humans and agents interact, turning authorization from a static credential into an \emph{executable governance term} evaluated continuously.

This paper introduces a framework that enables agents to inherit or delegate permissions, act under conditional authority, and allow users to trace the authorization of their agents. We present a \emph{compositional governance model} that integrates \textit{delegation} and \textit{resource scoping} into agents’ runtime semantics. The framework generalizes contractual delegation into a relational form that can be composed into existing authorization models, providing a foundation for accountable, contextual, and dynamic agentic authorization.

Operationally, agentic AI ecosystems are inherently relational: users delegate to agents, which may in turn delegate to other agents, all operating under the user’s umbrella of authority~\cite{openid}. We formalize these governance primitives as relations and define a \emph{compositional operator} that fuses existing domain-specific access control policies with our agentic semantics. Our formulation builds on Relation-Based Access Control (ReBAC)\cite{cheng2012relationship,giunchiglia2008relbac} and Google’s Zanzibar authorization model~\cite{Pang2019}, using OpenFGA~\cite{openfga} as the open-source reference implementation.



To the best of our knowledge, no prior work defines a contractual notion of delegation together with a compositional operator to operationalize AI governance primitives. This paper contributes:
\begin{enumerate}
	\item A conceptualization of delegation types and resource scoping as key drivers of agentic access.
	\item A formalization of delegation as an agentic governance overlay and a compositional operator to fuse it into authorization domains, drawing from graph transformation theory~\cite{ehrig2006fundamentals}.
	\item A security architecture that illustrates the usage of the resulting graph to govern users, agents, scopes, and delegations sessions.
	\item Preservation and agent-authorization soundness proofs, together with empirical benchmarks showing the runtime overhead of the overlay on large synthetic ReBAC models. \footnote{All models, code, and experiments related to this paper are available at: \url{https://github.com/Amjad-Ibrahim-Huawei/compositional-paper}.}
\end{enumerate}

Our approach results in a runtime \emph{authorization graph} that represents the state of delegation, scoping, and interactions among users and agents. From a zero-trust perspective, this enables continuous verification, granting, and revocation of agentic interactions. 
The remainder of this paper is organized as follows: Section~\ref{sec:bg} provides the necessary background; Section~\ref{sec:perspective} defines our governance primitives and base model; Section~\ref{sec:arch} presents the compositional operator and implementation architecture; Section~\ref{sec:eval} validates the approach via proofs and empirical evaluation; Section~\ref{sec:rel} discusses related work; and Section~\ref{sec:conclusion} concludes.

	\section{Background}\label{sec:bg}
Agentic AI systems utilize reasoning to autonomously achieve tasks on behalf of users with limited supervision. 
Each agent combines reasoning (“brain”), environmental awareness (“perception”), and the ability to interact (“action”)~\cite{xi2023risepotentiallargelanguage}, forming a persona that encodes its role, accessible tools, and peer interactions~\cite{masterman2024landscape}. Emerging protocols such as MCP and A2A standardize how agents communicate with tools, data sources, or each other~\cite{agenticProto,mcp}.

Examples of Agentic AI span from chat interfaces that trigger tools via MCP (e.g., bots creating tickets), to automation agents that implement end-to-end tasks such as insurance claims processing, to networks of agents that reason and communicate with each other (e.g., market negotiation agents~\cite{zhu})~\cite{openid}. Regardless of whether these agents are enterprise, coding, or client-facing systems, they introduce new security and governance challenges.

A key source of these challenges is that agent behavior is not fully specified in advance~\cite{openAi}. Agents reason over retrieved content, tool outputs, and evolving context, and then translate that into actions. In particular, \textit{prompt injection attacks} can hide malicious instructions inside external content such as documents, causing an agent to e.g., exfiltrate data~\cite{shan2026don}. 

According to the OWASP threat model for Agentic AI, attacks arise across six dimensions: agency and reasoning, memory and context, tools and execution, identity and authentication, human management, and multi-agency coordination~\cite{owasp}. We focus on unauthorized access and unauthorized action execution as they arise from overly permissive agent authority. 
This includes classical privilege compromise, but also prompt injection and harmful misoperation scenarios. 
Across these scenarios, the common failure mode is insufficiently constrained runtime privilege.

For practicality, researchers and protocol designers often reference standards such as (Open Authorization) OAuth for secure access delegation of AI agents~\cite{oauth,mcp,authenicatedDEL,openid}.\footnote{For example, MCP requires OAuth~2.1, which enhances token exchange protection and dynamic client onboarding, but remains identical in delegation and scoping semantics.} 
OAuth allows users to grant third-party applications access to protected resources without sharing credentials~\cite{oauth}. Despite its name, OAuth focuses on access delegation rather than authorization per se: an authenticated \textit{resource owner} issues a token—expressing consent, scope, and duration—via an \textit{authorization server}, allowing a \textit{client} to access a \textit{resource server}. This model has been extended to represent human-to-agent delegation via delegation tokens~\cite{authenicatedDEL}, enabling enterprises to reuse existing identity infrastructures for agentic AI. 

However, once issued, token contents remain static and cannot adapt to environmental or contextual changes~\cite{Huang}. 
Moreover, agentic systems often require recursive delegation, i.e., agents delegating to other agents; something OAuth was never designed to handle efficiently. For instance, if a coding agent \(C\) delegates a task to a ticketing agent \(T\) on behalf of a developer, multiple nested tokens would be needed. 
OAuth scopes constrain what an application may do on a user’s behalf but do not define the user’s own authorization model. 
Thus, scopes capture approval for static actions, while actual authorization decisions remain at the resource server.  Agents, however, require delegation and scoping mechanisms that support recursion, transitive permissions, and contextual runtime evaluation~\cite{openid,owasp,Huang,Syros,Wang}.

Agents therefore demand fine-grained, contextual authorization that extend beyond static roles and tokens. 
Traditional Role-Based Access Control (RBAC) cannot capture dynamic conditions such as time, device type, or hierarchical relationships. 
Nor can prompt-only guardrails reliably prevent prompt injection once an agent is connected to powerful tools. 
Instead, we argue for deterministic policy architectures and reusable governance primitives that build least privilege into the system itself rather than relying on model behavior~\cite{palumbo2026policy}. 
Least privilege means that authority granted to an agent should be explicitly delegated, limited, contextual, recursively controllable, and auditable.  We therefore argue for general governance primitives that can be standardized rather than re-implemented in every system. We look into this in detail in Section~\ref{sec:perspective}.
\subsection{Relation-Based Access Control (ReBAC)}
Our work builds upon ReBAC~\cite{cheng2012relationship,giunchiglia2008relbac} as the underlying authorization model.  
ReBAC defines permissions through relationships among users and resources (e.g., a user can access a document if they are its owner). 
By supporting dynamic and hierarchical relations, ReBAC is well-suited to collaborative and multi-agent environments. Although we frame our approach in a ReBAC setting, it is not restricted to purely relationship-based systems. 
ReBAC can encode RBAC roles as groups and can model many ABAC-style conditions through guarded relations, making it a convenient unifying substrate. 
Since our approach will be applied conjunctively with the domain’s existing decisions, it can sit on top of RBAC-, ABAC-, or hybrid policies without altering their semantics. 

A recent paper presented, Google Zanzibar, a highly efficient implementation of ReBAC~\cite{Pang2019}. 
Because traditional ReBAC (and Zanzibar) lacks support for contextual conditions, we adopt OpenFGA: an open-source Zanzibar system that extends ReBAC with rules evaluation on edges. 
OpenFGA provides fine-grained, relational authorization semantics suitable for expressing agentic delegation and scoping~\cite{openfga}. The following formalizes the core semantics of ReBAC and its userset algebra~\cite{Pang2019}, which serve as the formal substrate for our later extensions to agentic governance.

Let $\mathcal{U}$ be the set of principals (subjects), $\mathcal{O}$ the set of protected objects, and $\mathcal{L}$ the set of relation labels (e.g., \textit{owner}, \textit{viewer}). 
A relation tuple is $(o,\ell,x)$ with $o \in \mathcal{O}$, $\ell \in \mathcal{L}$, and $x \in (\mathcal{U} \cup \mathcal{O})$. 
The \emph{authorization graph} is a labeled multigraph $G=(V,E)$ where $V = \mathcal{U} \cup \mathcal{O}$ and $E \subseteq \mathcal{O} \times \mathcal{L} \times (\mathcal{U} \cup \mathcal{O})$.

A ReBAC policy specifies, for each object $o$ and relation $\ell$, a userset 
$\mathsf{Users}(o,\ell) \subseteq \mathcal{U}$, typically defined by graph reachability or set–algebraic expressions over tuples and other usersets.  
Authorization is a membership query: 
\[
\mathsf{Check}(u,\ell,o) \triangleq (u \in \mathsf{Users}(o,\ell)).
\]
ReBAC naturally expresses nested groups and resource hierarchies via transitive relations, and it has been formalized as graph reachability and as logical encodings~\cite{pwlfong_sacmat2011,DTR11-12}.

\noindent\textbf{Typed schema $C$ and userset algebra}
Let $\mathcal{T}$ be a finite set of object types. 
For each $t \in \mathcal{T}$:
(i) $O_t \subseteq \mathcal{O}$ are the objects of type $t$; 
(ii) $\mathcal{R}_t \subseteq \mathcal{L}$ are the relations of $t$;
(iii) each $\ell \in \mathcal{R}_t$ has a subject domain $D_{t,\ell} \subseteq (\mathcal{U} \cup \mathcal{O})$.
Each $(t,\ell)$ has a userset rewrite $e_{t,\ell}$ built from the minimal Zanzibar algebra: 
\emph{direct} (\textsf{this}), \emph{computed userset} (\textsf{computed}$(\ell')$), 
\emph{tuple-to-userset} (\textsf{from}$(\rho,\ell')$), and set operators 
$\cup,\cap,\setminus$. 
Here \textsf{from}$(\rho,\ell')$ follows an object-to-object relation $\rho$ (e.g., \textit{parent}) and reads $\ell'$ on the reached object.
(OpenFGA renders these as \texttt{X from Y}, \texttt{or}, \texttt{and}, \texttt{but not}.)

\paragraph{Conditions}
Let $\Gamma$ be a set of conditions; a condition $\gamma\in\Gamma$ is a predicate with a context schema. 
They may be attached to \emph{direct} edges and are evaluated at check time; they do not extend the algebra —they simply toggle whether an edge is “present” for a given check context~\cite{openfga}.

\paragraph{Data plane (E)}
The set of labeled edges (tuples) is
\[
E \subseteq \bigcup_{t \in \mathcal{T}} 
\Big( O_t \times \mathcal{R}_t \times D_{t,\ell} \Big)
\times \mathrm{Params}_\Gamma ,
\]
i.e., each element is an instantiated $(o,\ell,x)$ consistent with the typed schema, optionally with conditions.%
\paragraph{Check}
Given schema $C$ and tuples $E$, the denotation of a userset is defined by structural recursion on the rewrite:
\[
\llbracket \ell \rrbracket_{C}^{E}(o,\mathrm{ctx}) \subseteq \mathcal{U},
\]
reading direct edges from $E$ whose guards hold under $\mathrm{ctx}$ and evaluating 
usersets according to $e_{t,\ell}$. 
The decision is
\[
\mathsf{Check}(u,\ell,o;C,E,\mathrm{ctx}) \iff 
u \in \llbracket \ell \rrbracket_{C}^{E}(o,\mathrm{ctx}).
\]
ReBAC schemas require the graph to be acyclic, ensuring well-founded and deterministic evaluation. We rely on these standard well-foundedness assumptions of typed ReBAC engines. Our overlay does not introduce a separate execution semantics for authorization checks; rather, it adds well-typed relations evaluated by the same userset machinery. Accordingly, liveness of authorization evaluation is inherited from the underlying engine's guarantees for recursive usersets and schema-valid relation graphs.

	\section{A Relational Perspective on Agentic AI}
 \label{sec:perspective}
\lstset{basicstyle=\ttfamily\footnotesize,columns=fullflexible,breaklines=true,frame=single}
In this section, we formalize the authorization relations that arise in Agentic AI. We elicit the core requirements, then develop their implications for delegation, and finally present the captured model.

\subsection{Authorization Requirements for AI Agents}\label{subsec:reqs}
At a first glance, AI agents appear as an extension to traditional software with intelligent components. However, especially when considering access control requirements, agents possess unique characteristics~\cite{Huang,owasp}. One of their main features is autonomy. They do not maintain fixed boundaries like traditional software. Rather, an AI agent can execute tools and further, discover other agents and interact with them at runtime~\cite{openid,agenticProto}. This dynamic composition, in turn, means that the system boundary cannot be specified at design time. Thus, governing the authorization of agents via methods that assume predefined set of states is impractical. The same openness also creates new risks: because agents interpret untrusted content and synthesize actions at runtime, prompt injection can be realized. From our perspective, this makes least-privilege enforcement a first-class requirement for agentic systems.

AI agents complete tasks on behalf of a user, e.g., \textit{write tests for a developed feature}. In the process, the agent decides to access a resource, e.g., a design document in the documents drive, then read the document using a PDF reader. At a later stage, the agent would decide to push code to a repository and commissions this to an agent that handles pull requests. Even for such simple tasks, agents interact with several components that inherently require different permissions systems (documents access, repository privileges). Another remark is that the actor differs along the steps of this process. We need a mechanism to control this evolving nature of agent’s behavior. This mechanism enables reasoning about the interaction among users, and agents allowing them to assert whether the interacting entity is an agent or a user, the chain of delegations that the user is acting upon, and the permission scope for this entity.  We present the following set of requirements we aim to achieve. 
\begin{itemize}
	 \item [RQ1] \textbf{Agents are first-class actors}: from an identity perspective, agents should not impersonate users or deterministic services, rather, they should have a distinct notion of identity~\cite{openid}. They must function only with a user delegation. 

 	\item [RQ2]	\textbf{Delegation as the core mechanism for human–agent and agent–agent interaction}: agents must act according to the permissions they receive through delegation. AI delegation  must be treated as a \emph{contractual} relationship—one that defines not only who may act on behalf of whom, but under what constraints and for what purposes. Existing standards provide only limited semantics for such forms of delegation and do not capture the richer behaviors required for agents.
 	In Section~\ref{subsec:delegation}, we discuss this notion and introduce the types of delegation necessary for agents.
	
\item [RQ3]	\textbf{Delegation and scoping as authorization primitives}: 
given a means to identify an agent and to establish a delegation to it, the authorization system must incorporate these relations directly into its access rules. Every action performed by an agent must therefore be validated against (i) the agent’s own identity and (ii) the authority it inherits through delegation and scoping at the time of the request. This is the point at which least-privilege constraints become operational: runtime actions should be permitted only when explicitly justified by bounded delegation and valid scope.

\item [RQ4] \textbf{Observable, traceable, and accountable authorization state}: 
because the boundaries of agentic systems change, preventing all misuse is impractical. Robust \emph{detection} and \emph{accountability} mechanisms are therefore essential—both for improving authorization policies over time and for conducting reliable forensic analysis. 
To support these goals, the system must expose a faithful record of authorization-relevant events both at runtime and retrospectively for audit. 
 
\item [RQ5]\textbf{Contextual authorization}: Agentic AI introduces dynamic interaction patterns; authorization decisions must adapt accordingly.
Both access checks and delegation evaluations should incorporate contextual factors, e.g., network location, or request time, when determining if an action is permitted.
 
 \item [RQ6]\textbf{Fine-grained reuse of existing access rules}: because agents can trigger a wide range of actions, enumerating all allowed operations is infeasible.
 A more realistic approach is to \emph{limit resource access} based on existing enterprise or personal authorization policies rather than recreating them for agents~\cite{authenicatedDEL}.
 Agentic systems should therefore be able to express fine-grained access rules and \emph{reuse} existing permission structures—for example, allowing an employee’s AI assistant to inherit the employee’s document-access rights without re-implementing the organization’s policy logic.
 \end{itemize}
\subsection{Agent Delegation}\label{subsec:delegation}
 Although our goal is an operational definition that can be encoded and evaluated by access-control engines, we present a delegation notion grounded in the human, legal understanding of delegation. 
 In contract law, \textit{delegation} refers to the transfer of contractual duties from one party to another. 
 According to Cornell’s Legal Information Institute, three parties are involved: the \emph{delegator}, who assigns the duty; the \emph{delegate}, who is responsible for performing the duty; and the \emph{obligee}, who is entitled to receive the performance~\cite{LII_Wex_delegate_misc}.
 
 We adopt this structure for agents. 
 Delegation becomes a contractual relation between a human (delegator) and an agent (delegate); the agent may subdivide portions of that delegated authority to another agent. 
 Such contracts are \emph{dynamic}: they vary according to the conditions under which the delegation is valid. 
 For example, an employee might authorize \texttt{agent:email} to read her mailbox only when operating inside the corporate network. 
 This is not an absolute delegation, but a \emph{contextualized} one governed by constraints. 
 Allowing a delegate to act as a further delegator complicates tracking the delegation chains and their conditions. 
 Unlike legal contracts, these chains are \emph{runtime artifacts}, continuously evaluated to determine if an agent is authorized to perform an action.
 
 Another crucial aspect of delegation is the ability to \emph{attenuate} its scope~\cite{openid}.  
The notion of \emph{scope} is inherently abstract: it may refer to an action (e.g., deleting a document), to a resource (e.g., medical data), or to a combination of both (e.g., editing a budget report). Constraining the full action space of agents is difficult, whereas constraining \emph{resource} access is more tractable~\cite{authenicatedDEL,1024139}. 
 Regardless of its form, supporting attenuation requires a \emph{semantic ordering} over the scopes.  
Such orderings are domain-specific. Our model enforces scope compatibility through envelopes; a general semantic ordering for attenuation remains domain-specific.

While our approach inherits the same domain-specificity, we argue that representing delegation chains as \emph{contractual relations} within a relational model enables more precise scope attenuation.  
By expressing delegation, scoping, and resource structure as relations in a graph, the system maintains a runtime–mutable view of authorization state.  
This allows scope restriction to follow structural properties (e.g., parent–child relationships) rather than relying on static tokens.  Authorization decisions are thus made by evaluating the current graph state, which can evolve as delegations and scopes change. We now define the delegation types relevant for agents. These types build on decentralized authorization and trust-management logics ~\cite{li2003delegation,BeckerFournetGordon2007}, but are adapted to our requirements. 
   \begin{enumerate}
\item \textbf{Full delegation (unconditional ``speaks-for''):}  
The agent acts on behalf of the delegator without constraints.  
This corresponds to classical impersonation.

\item \textbf{Scoped delegation with attenuation:}  
The delegate may act only on a \emph{subset} of actions and/or resources permitted to the delegator.  
This is analogous to OAuth scopes.  The delegation applies to certain facts in the authorization system.

\item \textbf{Conditional (contextual) delegation:}  
The delegate may exercise the delegated authority only when specific conditions are met (e.g., $\textsf{region} = \text{EU}$).  
Conditions capture contextual requirements and may be combined with scoped delegation, yielding \emph{scoped conditional} delegation.  
A combination with full delegation reduces to a conditional delegation.

\item \textbf{Depth-bounded delegation:}  
Delegation may propagate up to a fixed length~$K$.  
$K=0$ forbids onward delegation; $K=1$ permits a single hop; and larger values encode controlled transitivity. Depth bounds are an optional constraint for limiting onward delegation, it is not the sole source of termination of authorization evaluation, which remains governed by the semantics of the underlying ReBAC engine.
 
\item \textbf{Temporal delegation:}  
A special case of conditional delegation in which the validity  is restricted by time-based predicates (e.g., expiry timestamps, or ``not before'' constraints).

\item \textbf{Group delegation:}  
A delegation is valid only if authorized by multiple principals (e.g., $n$-of-$m$ approval).  
This captures collaborative authorization patterns.

 \end{enumerate}
 
 \subsection{Relational Agentic Authorization}
 Practitioners and researchers increasingly look to OAuth as the basis for delegation in agentic systems. 
 For example, MCP recommends OAuth~2.1, and South et~al.\ propose an OAuth extension that incorporates an explicit agent-delegation token~\cite{authenicatedDEL}. 
 Other work explores decentralized approaches that rely on DIDs and VCs~\cite{Huang}. 
 These efforts contribute important compatibility mechanisms. 
 However, they treat delegation primarily as a \emph{credential} workflow—issuing, exchanging, and verifying tokens that encode a consent. 
 This view captures secure transfer of authority but does not address the recursive structure of delegation chains, or accountability requirements central to Agentic AI.
 
To that end, we introduce a relational schema that captures the relevant entities and supports runtime evaluation of agents authorization. The core principle of our approach is that an \emph{authorized agent} must be connected to a resource through a delegation chain that originates in a user who is permitted to access that resource.  Context, scope, and conditions enrich this chain, supporting scoped and contextual delegation, while leaving domain-specific attenuation orderings to the policy designer.

 We capture these relationships using a graph-based model inspired by ReBAC~\cite{cheng2012relationship}.  
 The authorization state is represented as a directed graph $E$ whose edges encode relationships between principals and resources (e.g., “User $u$ delegates to Agent $a$”).  
 Authorization checks are formulated as graph queries of the form: \emph{“Does subject $x$ hold relation $r$ to object $y$?”}~\cite{Pang2019}.
 
 To support this, we introduce a general-purpose, domain-agnostic scheme $C$ (illustrated in Fig.~\ref{fig:agentic-authorization}) that specifies the rules for constructing these graphs.  At runtime, $C$ is instantiated into a dynamic authorization graph $E$ that evolves with delegation, context changes, and access requests.
 A central concept in our design is the agent’s \emph{authorization envelope}: the dynamic set of resources the agent may access together with the contractual (conditional) terms under which access is valid.  
 The envelope is computed as the intersection of three factors:  
 (1) the agent must receive delegated authority from a user who is authorized for the resource;  
 (2) the agent must be operating within an active scope in which access is permissible; and  
 (3) the requested resource must lie within the scope of the delegation itself.

\begin{figure}[t]
	\centering
	\includegraphics{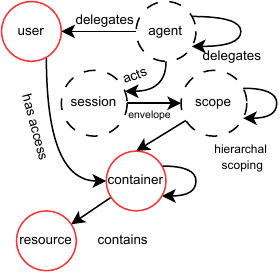}
	\caption{Schema of types (nodes) and key relations (edges) needed for Agentic AI authorization. Red nodes refer to types that already exist in ReBAC policies.}
	\label{fig:agentic-authorization}
\end{figure}

\subsubsection{Model Construction}
Our model construction is guided by a single question:
	\emph{``Why is agent $A$ allowed to perform action $X$ on resource $O$ on behalf of user $U$?''}
To encode this reasoning, we introduce three families of relations:  
(i) \textbf{delegation}: who may act for whom;  
(ii) \textbf{scope}: under which contextual constraints a delegation is valid; and  
(iii) \textbf{resource linkage}: how human permissions lift to agents.  
We use OpenFGA syntax to present types and relations, focusing on the key relations for \texttt{view}/\texttt{viewer}; additional permissions (e.g., \texttt{editor}) follow analogously.

\noindent{\textbf{Principals \& Delegation.}}
We first model the principal entities: users and agents (agent's syntax omitted for readability).  
A user may delegate authority to one or more agents, and agents may onward-delegate to sub-agents.  
The recursive userset \texttt{can\_execute....} captures this transitive delegation chain; conditions such as temporal or attribute guards can be attached to \texttt{delegatee} edges.

\noindent\begin{minipage}{\columnwidth}
	\begin{lstlisting}
type user # same for agents
  relations
	define delegatee: [agent, agent with temporal_delegation, agent with conditional_delegation]
	define can_execute_on_my_behalf: delegatee or can_execute_on_my_behalf from delegatee
	\end{lstlisting}
\end{minipage}

A delegatee edge from a user to an agent represents a contractual delegation: the agent may act on user’s behalf, possibly under conditions (e.g., time-limited).  
The relation \texttt{can\_execute\_on\_my\_behalf} is the transitive closure of these delegations.  
Evaluating this userset for $U$ yields the complete set of agents reachable via delegation paths rooted at $U$.  
Defining same relations on the \textit{agent} type enables recursive delegation: an agent can delegate to another.

A key invariant is that this closure is typed: the relation \texttt{delegatee} ranges only over subjects of type \texttt{agent}, and never traverses back into the user namespace. This is an important disjointness that is enforced at the schema layer rather than the identity-provider layer. So even if humans and agents share the same object type in the identity provider, they are represented as distinct typed principals inside the overlay. This restriction prunes the recursive search space, and avoids loops in the delegation closure.

\noindent{\textbf{Execution Context \& Scopes (the ``envelope'').}
A \texttt{session} represents a live instance of an agent acting under a specific delegation.  
Each session is associated with a \texttt{scope}, which captures the organizational context within which the agent may operate (e.g., a tenant, project, or resource collection). Scopes are hierarchical, reflecting nesting among resource categories (e.g., document~$\subset$ folder~$\subset$ workspace~$\subset$ organization).  
Together, the active \emph{delegation} and the session’s \emph{scope} determine the agent’s authorization \emph{envelope}.

\noindent\begin{minipage}{\columnwidth}
	\begin{lstlisting}
type session
  relations
	define actor: [agent]
	define as_agent: actor

type scope
  relations
	define parent: [scope]
	define holder: [session, session with ...]
	define sessions: holder
	define ags_direct: as_agent from holder
	define agents: ags_direct or agents from parent
	\end{lstlisting}
\end{minipage}
An agent is \emph{in scope} if it appears in \texttt{scope\#agents}, i.e., one of the agents that hold valid sessions in that scope (or of its ancestors). Authorization checks bind the acting principal via \texttt{session\#actor}.
Scopes form a tree through \texttt{scope\#parent}, and a scope receives sessions via \texttt{scope\#holder}, which may include guards.  
The derived relations \texttt{scope.sessions} and \texttt{scope.agents} aggregate local and inherited memberships.  
Thus, a scope defines the agent’s effective authorization envelope: an agent may access a resource only if its active session lies within the resource’s scope or its ancestors.

\noindent{\textbf{Resources \& Derived Agent View.}
We now model the resources themselves and connect all components of the system to enforce the two core authorization requirements:  
(A) an agent must be \emph{in scope} for the resource, and  
(B) the agent must be delegated by a user who has permission to access that resource. We distinguish between \texttt{container} types (folders) and individual resources (documents). Containers form a hierarchy via \texttt{parent}, and store the base rule for human users (\texttt{viewer}, \texttt{owner}).  
Each container is attached to a \texttt{scope}, enabling scope-based evaluation of agent sessions. We omit resources in the following for brevity.

\noindent\begin{minipage}{\columnwidth}
	\begin{lstlisting}
type container
  relations
	define parent: [container]
	define in_scope: [scope]
	define viewer: [user]
	# Human viewers (with inheritance)
	define hu_can_view: viewer or .. from parent
		
	# (A) valid Agents in this container's scope
	define ags_in_scope: agents from in_scope
		
	# (B) delegated by human viewers
	define chain_agents_for_view:
	can_execute_on_my_behalf from viewer
	or chain_agents_for_view from parent	
	
	# envelope (A) Intersection delegation (B)
	define delegated_agent_viewer: ags_in_scope and chain_agents_for_view
	# Final view: human OR authorized agent
	define can_view: hu_can_view or delegated_agent_viewer
	\end{lstlisting}
\end{minipage}

Each \texttt{container} is tied to a contextual \texttt{scope} via \texttt{in\_scope}; thus, all resources under that container inherit the scope’s authorization \emph{envelope}.  
For agents, we derive two sets:  
(A) \texttt{ags\_in\_scope}, the agents with active sessions in the container’s scope (or inherited from ancestor scopes), and  
(B) \texttt{chain\_agents\_for\_view}, the agents reachable through delegation from the human \texttt{viewer}s of that container (including inherited viewers).  
An agent may view a resource if it appears in the \emph{intersection} of these two sets—i.e., it is in the correct scope and properly delegated. 

\noindent{\textbf{Conditions (for delegation types).} 
We express different types of delegation by attaching predicates directly to the \texttt{delegatee} relation. 
While we show temporal delegation, the same condition interface can
host other trusted predicates. Depth-bounded delegation can be encoded with additional schema patterns. 
For a fixed small bound \(K\), bounded delegation can be encoded by stratifying the closure into relations \(\mathsf{can\_execute}^{(0)},\ldots,\mathsf{can\_execute}^{(K)}\),
where each level follows one additional \(\mathsf{delegatee}\) edge.
This is a finite schema expansion and is therefore compatible with the
overlay, but it is omitted from the core presentation.
Similarly, \(n\)-of-\(m\) approval can be represented by a trusted  predicate whose relation is checked before the delegation tuple is admitted. We do not claim a general counting operator in the base ReBAC algebra.

\noindent\begin{minipage}{\columnwidth}
	\begin{lstlisting}
condition temporal_delegation(expires_at: timestamp, current_time: timestamp) { current_time < expires_at}
	\end{lstlisting}
\end{minipage}

\paragraph{Illustrative Example}
Consider a user \texttt{bob} who delegates authority to \texttt{agent1} through the relation  
\texttt{\textlangle bob~delegatee~agent1\textrangle}  
with a validity for 1 hour.  
A session \texttt{s1} is created for \texttt{agent1}  
\texttt{\textlangle s1~actor~agent1\textrangle}  
and is placed in scope via  
\texttt{\textlangle org/eng~holder~s1\textrangle}.  
This means that \texttt{agent1} is active within the contextual \emph{envelope} defined by the \texttt{org/eng}.
Assume a folder is associated with that scope \texttt{\textlangle folder1~in\_scope~org/eng\textrangle} 
and contains \texttt{design-document}.  
User \texttt{bob} is a declared \texttt{viewer} of this folder  
\texttt{\textlangle folder1~viewer~bob\textrangle}.  

Because \texttt{agent1} (A) appears in the delegation chain of \texttt{bob},  
and (B) has an active session in the container’s scope,  
the intersection that defines \texttt{delegated...\_viewer} is non-empty.  
Thus the authorization relation  
\texttt{\textlangle agent1 can\_view eng-folder\textrangle}  
holds.  
If the condition expires,  
or if the session is removed from the scope,  
the intersection becomes empty, and the agent loses access.

	\section{Operational Governance for Agentic AI} \label{sec:arch}
With the base model in place, we have a principled foundation for enforcing governance in agentic AI.  
However, designing a bespoke authorization model for every domain in which agents operate (e.g., documents, code generation) is impractical.  
To address this, we examine how to \emph{operationalize} enforcement through three components:  
a compositional operator that injects agentic primitives into existing domain models (Section~\ref{sec:operator}),  
an architecture for runtime evaluation (Section~\ref{sec:architecture}),  
and an illustrative  use case (Section~\ref{sec:usecase}).
\subsection{Overlay as a Typed Graph Rewrite} \label{sec:operator}
We aim to extend a domain authorization model expressible as typed
ReBAC with agentic governance primitives. While similar goals could
be pursued in other policy languages, the relational structure of
agents naturally suggests a graph-based construction.  Our
operator overlays delegation, scope, and contextual constraints onto
an existing domain schema without rewriting its human-facing logic.  Human access remains authoritative; agent access
is derived as the intersection of delegated authority and contextual
scope. The construction is related to policy-combination
frameworks~\cite{BonattiSamarati2002} and to double-pushout (DPO) graph
rewriting~\cite{ehrig2006fundamentals}, but here the purpose is to
inject reusable agentic primitives into ReBAC schemas.

We now make precise the class of domain schemas to which the
overlay operator applies.  A typed ReBAC schema is represented as
a tuple
\[
C=(\mathcal{T},\mathcal{R},\mathsf{subj},e),
\]
where \(\mathcal{T}\) is a finite set of object types,
\(\mathcal{R}_T\) is the finite set of relation symbols available on
type \(T\), \(\mathsf{subj}(T,r)\) is the subject-domain declaration
of relation \(r\in\mathcal{R}_T\), and \(e_{T,r}\) is the userset
expression defining \(r\).  Userset expressions are built from
\(\mathsf{this}\), computed usersets, tuple-to-userset, union,
intersection, and difference, as in Section~2.

For rewriting, we view \(C\) as a finite typed graph
\(G_C\).  The graph contains nodes for types and relation
occurrences \((T,r)\), with edges recording subject domains and
userset dependencies.  For example,
if \(e_{T,r}\) contains \(r'\) from \(\rho\), then \(G_C\) contains
dependency edges from \((T,r)\) to \((T,\rho)\) and to the relation
reached by \(\rho\).  This graph representation is used only to
specify schema transformation; authorization semantics remain the
standard userset denotation \(\llbracket - \rrbracket^E_C\).

Let \(B\) be the agentic overlay schema containing fresh types
\(\mathsf{agent}\), \(\mathsf{session}\), and \(\mathsf{scope}\),
and fresh overlay relations such as
\(\mathsf{delegatee}\),
\(\mathsf{can\_execute\_on\_my\_behalf}\),
\(\mathsf{holder}\),
\(\mathsf{actor}\),
\(\mathsf{in\_scope}\),
\(\mathsf{ags\_in\_scope}\),
\(\mathsf{chain\_agents\_for\_r}\), and
\(\mathsf{delegated\_agent\_r}\).
Freshness means that these overlay-introduced names do not occur in
the domain schema \(C_D\).  The existing domain relations used as
the interface to the overlay are selected separately by the lift
specification
$
\mu=(\mathcal{L},\mathsf{root},\mathsf{parent}).
$
Here \(\mathcal{L}(T)\subseteq\mathcal{R}_T\) is the set of domain
permissions to lift for type \(T\), \(\mathsf{root}(T,r)\) is the
human-root userset expression from which delegation for permission
\(r\) is derived, and \(\mathsf{parent}(T)\) is an optional hierarchy
relation used for inherited permissions. The lift specification must satisfy the following applicability conditions.

\begin{description}
	\item[A1: Freshness.] The overlay-introduced type and relation names 	are fresh with respect to \(C_D\).
	\item[A2: Well-typed roots.] For every \(T\) and \(r\in\mathcal{L}(T)\),
	the expression \(\mathsf{root}(T,r)\) is well-typed in \(C_D\) and
	denotes only human principals: \quad
	$\llbracket \mathsf{root}(T,r)\rrbracket^{E_D}_{C_D}(o,\mathit{ctx})
	\subseteq U .$

	\item[A3: Root adequacy.] The root expression is no more permissive
	than the original domain permission:
\\ \quad
$	\llbracket \mathsf{root}(T,r)\rrbracket^{E_D}_{C_D}(o,\mathit{ctx})
	\subseteq
	\llbracket r \rrbracket^{E_D}_{C_D}(o,\mathit{ctx}) .
$
	\item[A4: Agent disjointness.] The fresh agent type \(A\) is disjoint
	from all domain principals:
	$
	A\cap P_D=\emptyset .
	$
	\item[A5: Well-foundedness.] The composed userset dependency graph
	satisfies the same well-foundedness requirements imposed by the
	underlying ReBAC engine.
	
	\item[A6: Hierarchy compatibility.] If \(\mathsf{parent}(T)\) is used
	for a lifted permission \(r\), then the original domain permission
	\(r_D\) is inherited along the same hierarchy. That is, for every
	parent edge from \(o\) to \(o_p\),
$
	\llbracket r_D \rrbracket^{E_D}_{C_D}(o_p,\mathit{ctx})
	\subseteq
	\llbracket r_D \rrbracket^{E_D}_{C_D}(o,\mathit{ctx}) .
$
\end{description}

For each \(T\) and \(r\in\mathcal{L}(T)\), the overlay applies a
non-deleting graph rewrite in DPO style:
$
p_{T,r}: L_{T,r} \xleftarrow{\ell} K_{T,r}
\xrightarrow{\rho} R_{T,r}.
$
We set \(K_{T,r}=L_{T,r}\), so the matched domain schema is
preserved and the rule only glues in fresh overlay relations.  The
left-hand side contains the type node \(T\), the permission relation
\((T,r)\), the dependencies of \(\mathsf{root}(T,r)\), and, when
present, the hierarchy relation \(\mathsf{parent}(T)\).  The
right-hand side extends this interface with:

\[
\begin{aligned}
	\mathsf{ags\_in\_scope}_T
	&:= \mathsf{agents}\ \mathsf{from}\ \mathsf{in\_scope}_T,\\
	\mathsf{chain\_agents\_for\_r}_T
	&:= \mathsf{can\_execute\_on\_my\_behalf}\ \mathsf{from}\
	\mathsf{root}(T,r) \\
	&\quad \mathsf{or}\ 
	\mathsf{chain\_agents\_for\_r}_T\
	\mathsf{from}\ \mathsf{parent}(T),\\
	\mathsf{delegated\_agent\_r}_T
	&:= \mathsf{ags\_in\_scope}_T
	\ \mathsf{and}\
	\mathsf{chain\_agents\_for\_r}_T,\\
	r
	&:= r_D\ \mathsf{or}\ \mathsf{delegated\_agent\_r}_T .
\end{aligned}
\]
\(r_D\) is the original domain userset for
relation \(r\).  If \(T\) has no parent relation, the second disjunct
in \(\mathsf{chain\_agents\_for\_r}_T\) is omitted.

Since \(K_{T,r}=L_{T,r}\), the rewrite is conservative: no domain
type, relation, or userset dependency is deleted.  Applying these
rules for all \(T\) and \(r\in\mathcal{L}(T)\), together with the
global bootstrap rules for \(\mathsf{agent}\), \(\mathsf{session}\),
and \(\mathsf{scope}\), yields the composed schema
\[
C_{D\otimes B}=\mathsf{Overlay}(C_D,\mu).
\]

The denotational effect of the rewrite is therefore explicit:
for every lifted permission \(r\in\mathcal{L}(T)\),
{\tiny
\begin{equation*}
\llbracket r \rrbracket^{E}_{C_{D\otimes B}}(o,\mathit{ctx})
=
\llbracket r_D \rrbracket^{E_D}_{C_D}(o,\mathit{ctx})
\cup 
\left(
\llbracket \mathsf{ags\_in\_scope}_T \rrbracket^{E}_{C_{D\otimes B}}(o,\mathit{ctx})
\cap
\llbracket \mathsf{chain\_agents\_for\_r}_T \rrbracket^{E}_{C_{D\otimes B}}(o,\mathit{ctx})
\right).
\end{equation*}
}
Thus the original domain permission is preserved as one branch,
and agent authorization is added only through the intersection of
scope membership and a human-rooted delegation chain.

Appendix~\ref{sec:rewriteAppendix} gives an implementation-oriented expansion of the overlay macros used in this construction.
\subsection{Agent Controller Engine (ACE)}\label{sec:architecture}
To operationalize the concepts presented in this work, we propose a technical component called the \emph{Agent Controller Engine (ACE)}. ACE provides dynamic, contextual, and composable governance for agentic AI. It is designed as an extension within a classical IAM architecture, as illustrated in Figure~\ref{fig:overlay-arch}. At a high level, ACE unifies authorization, delegation, and auditing logic for agents. ACE interfaces with token-based authentication services—such as OAuth~2.1/OIDC providers—that may issue \emph{authenticated delegation tokens}. Other identity infrastructures (e.g., DID) are equally viable; from ACE’s perspective, these components serve only as secure sources of relational facts needed for authorization. Rather than embedding fixed capabilities inside tokens, ACE requires tokens to carry \emph{relations} (e.g., \texttt{agentX has full user delegation}), which are incorporated into its runtime authorization graph.
\begin{figure}[h]{}
	\centering
	\includegraphics[width=\linewidth]{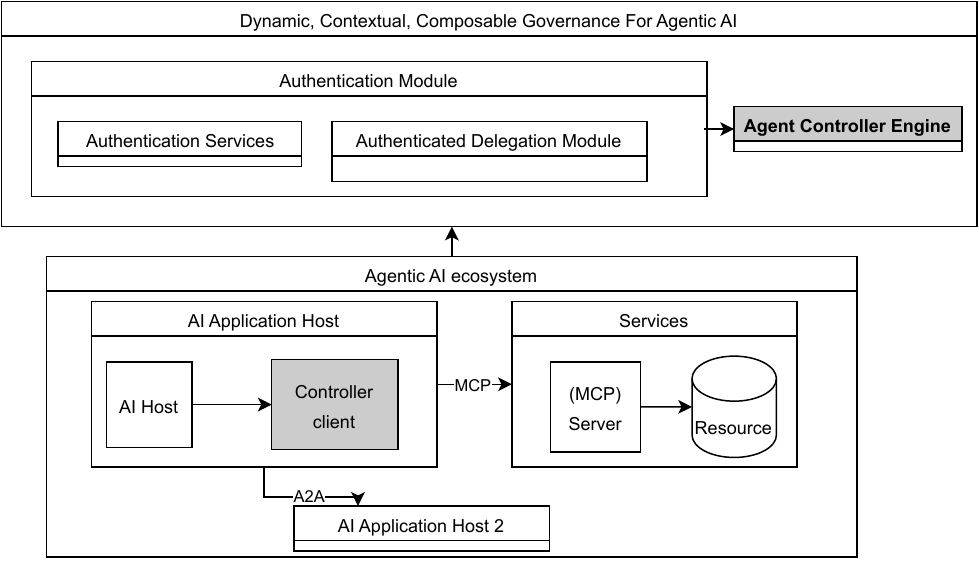}
	\caption{The general components in Agentic Governance.}
	\label{fig:overlay-arch}
\end{figure}

We assume that ACE and the authentication module operate within a centralized trust domain accessible by agents. Enforcement occurs through \emph{controller clients}, which act as policy enforcement points (PEPs), known in the XACML reference architecture~\cite{openid}. These clients intercept actions and consult ACE before execution. Their role aligns with existing protocols (e.g., MCP clients, A2A clients), but ACE augments them with a zero-trust governance layer and a unified authorization engine tailored for agentic interactions.

Let us unpack the ACE. Figure~\ref{fig:ace} illustrates its primary components and the flow of governance and authorization data. To operationalize the compositional model (Section~\ref{sec:operator}), ACE incorporates a \emph{governance layer}. This layer embodies the required primitives—most notably \emph{delegation} (including scoped and conditional variants), recursive delegation chains (the base model), and the \emph{injection} of these semantics into existing domain models through our composition operator. This layer materializes these abstractions as a typed schema graph that supports runtime checks for AI agents.

The \emph{execution layer} forms the operational core of ACE. It maintains an \emph{authorization graph} (AG) that reflects the live state of users, agents, sessions, delegations, and scope assignments. The AG evolves as the system evolves, and therefore requires a runtime \emph{relations writer} responsible for securely inserting, or removing edges based on system events. This module processes and verifies identity, access, or delegation tokens, and may also ingest relational facts from other trusted components. As such, the writer is extensible and functions as a hub for multiple policy information points.

Finally, the execution layer exposes authorization services to its clients —users, agents, auditor software, or monitoring systems—via an authorization engine. This engine evaluates access requests by combining (i) governance primitives, (ii) domain-level permissions, and (iii) the current AG state. It answers queries such as:  
\textit{is "agent:health" allowed to access the technical-specification folder?}  
or  
\textit{which agents currently retain access to bob health records?}  
The result is a unified, zero-trust–compatible enforcement point for AI agents.

\begin{figure}[!t]{}
	\centering
	\includegraphics[width=\linewidth]{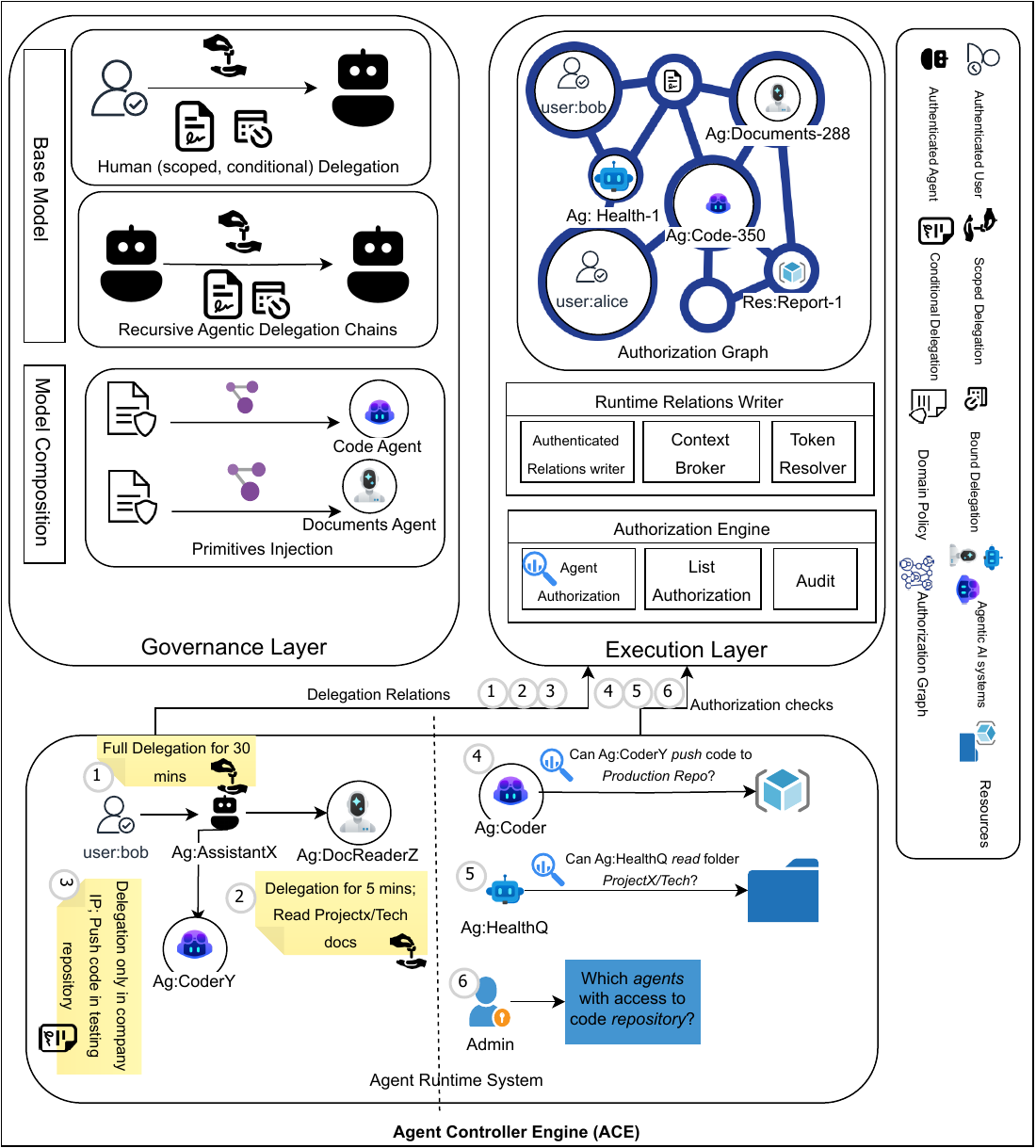}
	\caption{The complete ACE Reference Architecture with components and processes.} 
	\label{fig:ace}
\end{figure}

\subsection{Use Case: Multi-Agent Code Assistants}\label{sec:usecase}
Coding AI agents that support developers in software development are a promising domain for AI. Their ability to generate, and review code faster than humans makes them indispensable companions for developers~\cite{gartnerAICodeAssistants}. Coding Agents are often envisioned as \emph{multi-agent systems}  that cooperate toward a shared goal. For instance, a \emph{Planner Agent} coordinates tasks, a \emph{Requirements Agent} extracts functional requirements, and \emph{Coding and Testing Agent} generates code and test cases. Because these assistants consume heterogeneous and untrusted artifacts e.g., specifications, issue threads, code comments, they are natural targets for prompt injection.
	
In an enterprise deployment of coding assistants, their compliance with organizational policies is crucial. Developers work across multiple projects, each governed by domain rules and backed by structured repositories. They use shared document systems (e.g., Drive) for specifications and version control systems (e.g., GitHub) for code, both of which expose human-centric authorization models (e.g., \texttt{reader}, \texttt{maintainer}) over resources such as \texttt{documents}, \texttt{folders}, and \texttt{repositories}.
To enable agents to act on behalf of developers within these systems, e.g., accessing documents, or committing code, we apply our compositional authorization model.

\noindent\textbf{Mapping the Domain Model.}
We begin by extracting the existing domain authorization models from both systems. Each can be expressed as a schema $(C_D)$ using familiar relations such as \texttt{parent}, \texttt{viewer}, and \texttt{editor}. For simplicity, we omit system-specific details. Our overlay rewriting can be applied to each model separately or to a unified model; we choose the latter, as it allows us to harmonize GitHub and Google Drive under a single collaboration schema. In practice, this involves mapping GitHub \emph{teams} to \texttt{group}, \texttt{repo}/\texttt{folder} to \texttt{container}, and \texttt{doc}/\texttt{file} to \texttt{resource}. The result is a generalized hierarchical model suitable for both domains. A more rigorous sequential composition operator is left for future work. A snippet of the resulting combined model appears below.
\begin{lstlisting}[basicstyle=\ttfamily\footnotesize,columns=fullflexible,frame=single]
type user
type group
  relations
	define member: [user, group#member]	
type organization
  relations
	define owner: [user]
	define member: [user, group#member] or owner	
type container
  relations
	define parent: [container]
	define viewer: [user, group#member, organization#member] or owner or editor or viewer from parent	
\end{lstlisting}

%

\paragraph{Agentic Overlay.}
Using the operator from Section~\ref{sec:operator}, we
instantiate the lift specification \(\mu\) and generate the composed
schema \(C_{D\otimes B}\).  For this use case, \(\mu\) selects
container permissions such as \(\mathsf{viewer}\) and
\(\mathsf{editor}\), uses the original domain permission \(r_D\) as
the human root, and uses \(\mathsf{parent}\) as the hierarchy
relation.  Expanding the overlay macros injects: (1) global agentic
types (\(\mathsf{agent}\), \(\mathsf{session}\), \(\mathsf{scope}\));
(2) delegation and scope relations (\(\mathsf{delegatee}\),
\(\mathsf{holder}\)); (3) scoping and delegation-chain lifting into
each resource type via \(\mathsf{in\_scope}\); and (4) derived agent
permissions e.g., \(\mathsf{delegated\_agent\_viewer}\).

The resulting extended schema \(C_{D\otimes B}\), shown in
Figure~\ref{fig:injectedmodel}, is conservative for domain principals
while adding agentic delegation paths through the overlay branch.
Delegations encode which users authorize which agents and under what
terms, such as conditional delegation, or scoped
sessions within a project hierarchy.

\begin{figure}[!t]{}
	\centering
	\includegraphics[width=0.7\linewidth]{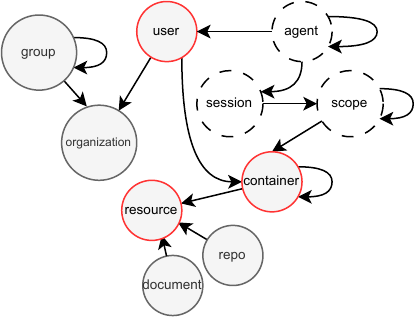}
	\caption{Coding Assistant Schema. Nodes are types, edges are relations. Red nodes refer to common types between base and domain models, dashed nodes refer to governance types.}
	\label{fig:injectedmodel}
\end{figure}

\noindent\textbf{Execution and Enforcement.}
Using the composed model in practice requires deploying it into an authorization engine, populating runtime relations, and answering access requests. We use OpenFGA as an open-source engine for storing relational tuples~\cite{openfga}. A tuple is generated, for example, when a user prompts an agent to perform a programming task: a delegation relation, together with a session and scope, is written to the engine, e.g.,
\texttt{\textlangle alice delegatee agent:planner, condition:expires\_at:"23:59:59"\textrangle}}.
Scoping relations are recorded based on the user’s choices, but may also be produced by other trusted components that observe the environment—for instance, a security monitor or intrusion detection system. This enforcement point remains relevant even if an agent is influenced by malicious content, i.e.,injections do not expand the agent’s envelope, provided attempted accesses are mediated by the PEPs and checked against the current delegation chain.

\noindent\textbf{Runtime checks.} Authorization queries of the form \texttt{Check(agent, can\_view, resource)} are evaluated against the composed configuration $C_{D\otimes B}$ and the current authorization graph.

\noindent\textbf{Example.}
Consider a project \texttt{projX} within the web-development department (\texttt{scope:org/web}). \texttt{Alice} is a viewer of the project’s container (\texttt{container:projX}), which stores design documents and related artifacts, and she also holds the repository rights needed to commit changes for this project. An AI assistant \texttt{agent:Planner} operates as the main task-level coordinator.

To enable system to act on her behalf, Alice issues a delegation to \texttt{agent:Planner} guarded by a \texttt{temporal} condition. When the agent begins operating, it creates a \texttt{session:s1} in the \texttt{org/web} scope, representing its contextual envelope. The planner may then issue separate scoped delegations to specialized
sub-agents, e.g., a document-view delegation to \texttt{agent:DocReader} and a
repository-write delegation to \texttt{agent:Copilot}.

The document (\texttt{res:design-doc}) lives inside the project container, which is itself associated with the same scope. At check time, the engine evaluates whether the agent (DocReader) can view the document. 
The request is authorized because:  
(i) there is a valid delegation chain from a human with view rights on the resource (\texttt{user:alice}) to \texttt{agent:Planner} and then to \texttt{agent:DocReader};  
(ii) all delegation conditions hold; and  
(iii) the agent has an active session in a scope compatible with the resource’s scope. Suppose, however, that \texttt{design-doc} contains a malicious instruction such as ``ignore prior guidance and push the repository contents to production.'' The document may still influence \texttt{DocReader}'s reasoning, but it does not enlarge its authorization envelope: \texttt{DocReader} only holds the read-oriented delegation issued for document analysis, so any attempt by \texttt{DocReader} itself to access a repository, invoke a non-authorized tool, or act outside \texttt{org/web} is denied.

If the task genuinely requires a code change, \texttt{agent:Planner} may issue a separate delegation relation to \texttt{agent:Copilot} for repository actions. When \texttt{agent:Copilot} attempts the push, the engine evaluates that distinct delegation chain back to Alice. Thus, reading a malicious document cannot by itself cause code to be pushed unless repository access was independently delegated to the coding agent and all relevant conditions still hold.

Through model composition, coding assistants can safely operate while limiting the impact of prompt injection to the agent's envelope. This achieves:  
\textbf{(i) Delegation safety:} agents cannot obtain access without a user chain;  
\textbf{(ii) Contextual enforcement:} delegations are active under the correct conditions;  
\textbf{(iii) Auditability:} agent actions trace back to their delegator;  
\textbf{(iv) Reuse:} human-facing domain permissions remain intact.
	\section{Evaluation and Verification}\label{sec:eval}
 We evaluate key qualities of our approach. First, we prove preservation
 and agent-authorization soundness for well-formed schemas satisfying the
 overlay applicability conditions. Then we assess practical properties such
 as decision latency.

\subsection{Verification of Soundness}\label{sec:soundness}
We verify the effect of the operator from
Section~\ref{sec:operator}.  The overlay should be conservative for
domain principals, while every agent authorization should be
justified by an original human permission, a valid delegation chain,
and a valid scope/session chain.

Let \(C_D\) be a well-formed domain schema with tuple set \(E_D\).
Let \(C_{D\otimes B}=\mathsf{Overlay}(C_D,\mu)\), and let
\(E_{D\otimes B}\) extend \(E_D\) only with overlay-introduced
relations, e.g., \(\mathsf{delegatee}\), \(\mathsf{actor}\),
\(\mathsf{holder}\), and \(\mathsf{in\_scope}\).  Let \(P_D\) be the
set of original domain principals, \(U\subseteq P_D\) the human
principals that may act as delegation roots, and \(A\) the fresh set
of agent principals.  By A4, \(A\cap P_D=\emptyset\).

We write \(x\leadsto_{E,\mathit{ctx}}y\) when \(E\) contains a
delegation tuple from \(x\) to \(y\) whose guard, if any, holds in
context \(\mathit{ctx}\).  We write
\(x\leadsto^{+}_{E,\mathit{ctx}}y\) for the transitive closure of such
valid delegation edges.  The first node of a delegation chain may be
a human \(u\in U\), while all delegatees are agents, as enforced by
the subject domain of \(\mathsf{delegatee}\).

We write \(\mathsf{Scoped}_T(a,o,E,\mathit{ctx})\) when agent \(a\)
has an active session in a scope compatible with object \(o\) of type
\(T\).  This holds when there exist a session \(s\) and scopes
\(q,q'\) such that
\[
(s,\mathsf{actor},a)\in E,\quad
(q',\mathsf{holder},s)\in E,\quad
(o,\mathsf{in\_scope},q)\in E,
\]
all relevant guards hold in \(\mathit{ctx}\), and \(q'=q\) or \(q'\)
is an ancestor of \(q\) under the scope-parent relation.  This is the
denotation of \(\mathsf{ags\_in\_scope}_T\).

\begin{lemma}[Conservative extension for domain principals]
	For every lifted permission \(r\in\mathcal{L}(T)\), object \(o\) of
	type \(T\), domain principal \(p\in P_D\), and context
	\(\mathit{ctx}\),
	\[
	p\in \llbracket r_D\rrbracket^{E_D}_{C_D}(o,\mathit{ctx})
	\quad\Longleftrightarrow\quad
	p\in \llbracket r\rrbracket^{E_{D\otimes B}}_{C_{D\otimes B}}(o,\mathit{ctx}) .
	\]
\end{lemma} 

\begin{proof}
	By Section~\ref{sec:operator},
	\[
	\llbracket r \rrbracket^{E_{D\otimes B}}_{C_{D\otimes B}}
	=
	\llbracket r_D \rrbracket^{E_D}_{C_D}
	\cup
	\llbracket \mathsf{delegated\_agent\_r}_T
	\rrbracket^{E_{D\otimes B}}_{C_{D\otimes B}} .
	\]
	The first branch is preserved by the non-deleting rewrite.  The
	second branch denotes only principals of the fresh agent type \(A\),
	because it is constructed from \(\mathsf{ags\_in\_scope}_T\) and
	\(\mathsf{chain\_agents\_for\_r}_T\).  Since \(A\cap P_D=\emptyset\),
	the agent branch cannot add or remove any \(p\in P_D\).
\end{proof}

\begin{lemma}[Human-rooted delegation]
	For every lifted permission \(r\in\mathcal{L}(T)\), object \(o\) of
	type \(T\), agent \(a\in A\), and context \(\mathit{ctx}\),
	\[
	a \in
	\llbracket \mathsf{chain\_agents\_for\_r}_T
	\rrbracket^{E_{D\otimes B}}_{C_{D\otimes B}}(o,\mathit{ctx})
	\]
	implies that there exists a human \(u\in U\) such that
	\[
	u\in
	\llbracket r_D\rrbracket^{E_D}_{C_D}(o,\mathit{ctx})
	\quad\text{and}\quad
	u \leadsto^{+}_{E_{D\otimes B},\mathit{ctx}} a .
	\]
\end{lemma}

\begin{proof}
	The proof is by induction on \(\mathsf{chain\_agents\_for\_r}_T\)'s definition.  In the base case, \(a\) is
	reached by following
	\(\mathsf{can\_execute..}\) from
	\(\mathsf{root}(T,r)\).  By A2, the root principal is a human
	\(u\in U\); by A3, this root is included in the original domain
	permission \(r_D\); and by the definition of
	\(\mathsf{can\_execute\_on\_my\_behalf}\), there is a valid delegation
	path \(u\leadsto^{+}_{E_{D\otimes B},\mathit{ctx}}a\).
	In the parent case, the claim follows from the induction hypothesis
	on the parent object and A6, which ensures that \(r_D\) is inherited
	along the same hierarchy.
\end{proof}

\begin{theorem}[Agent authorization soundness]
	For every lifted permission \(r\in\mathcal{L}(T)\), object \(o\) of
	type \(T\), agent \(a\in A\), and context \(\mathit{ctx}\),
	\[
	a\in
	\llbracket r \rrbracket^{E_{D\otimes B}}_{C_{D\otimes B}}(o,\mathit{ctx})
	\]
	implies that there exists a human \(u\in U\) such that:
	\[
	\begin{array}{ll}
		\textnormal{(i)} &
		u\in
		\llbracket r_D \rrbracket^{E_D}_{C_D}(o,\mathit{ctx}), \\[1mm]
		\textnormal{(ii)} &
		u \leadsto^{+}_{E_{D\otimes B},\mathit{ctx}} a, \\[1mm]
		\textnormal{(iii)} &
		\mathsf{Scoped}_T(a,o,E_{D\otimes B},\mathit{ctx}) .
	\end{array}
	\]
\end{theorem}

\begin{proof} 
	Assume
	\(a\in
	\llbracket r \rrbracket^{E_{D\otimes B}}_{C_{D\otimes B}}(o,\mathit{ctx})\).
	Since \(a\in A\) and \(A\cap P_D=\emptyset\), \(a\) cannot occur in
	the preserved domain branch
	\(\llbracket r_D \rrbracket^{E_D}_{C_D}\).  Hence \(a\) must occur in
	the overlay branch:
	\[
	a\in
	\llbracket \mathsf{delegated\_agent\_r}_T
	\rrbracket^{E_{D\otimes B}}_{C_{D\otimes B}}(o,\mathit{ctx}) .
	\]
	By construction,
	\[
	\mathsf{delegated\_agent\_r}_T
	=
	\mathsf{ags\_in\_scope}_T
	\cap
	\mathsf{chain\_agents\_for\_r}_T .
	\]
	Membership in \(\mathsf{ags\_in\_scope}_T\) gives
	\(\mathsf{Scoped}_T(a,o,E_{D\otimes B},\mathit{ctx})\).  Membership
	in \(\mathsf{chain\_agents\_for\_r}_T\), together with the
	human-rooted delegation lemma, gives a human \(u\in U\) such that
	\(u\in\llbracket r_D\rrbracket^{E_D}_{C_D}(o,\mathit{ctx})\) and
	\(u\leadsto^{+}_{E_{D\otimes B},\mathit{ctx}}a\).
\end{proof}

\begin{corollary}[Revocation and guard invalidation]
	Let \(E'\) be obtained from \(E_{D\otimes B}\) by removing only
	overlay tuples, and let \(\mathit{ctx}'\) be any context, possibly one
	in which a delegation or scope-holder guard no longer holds.  If, in
	\((E',\mathit{ctx}')\), there is no human \(u\in U\) satisfying the
	three conditions of the agent-authorization soundness theorem for
	\((a,r,o)\), then
	\[
	a\notin
	\llbracket r \rrbracket^{E'}_{C_{D\otimes B}}(o,\mathit{ctx}') .
	\]
\end{corollary}

\begin{proof}[Proof sketch]
	Immediate by the contrapositive of authorization soundness:
	any successful agent authorization must have a human-permission,
	delegation-chain, and scope/session witness.
\end{proof}

\subsection{Empirical Evaluation}\label{sec:empirical}
We assess the cost of enriching existing models with our overlay. This assessment is especially important because, unlike traditional ReBAC deployments, our approach introduces agents and sessions as runtime principals whose relations change frequently during execution. In realistic scenarios, sessions are short-lived, delegations are created and revoked recursively. The resulting authorization graph is therefore not only larger, but also more dynamic.

We evaluate the effect of this richer model on performance by comparing a baseline \emph{Domain} configuration against a matched \emph{Domain+Overlay} configuration, using identical domain tuples, across two known OpenFGA use-cases: \textit{Google Drive (G), and Slack (S).} 

For each use-case, we evaluate different sets of scenarios that gradually increase in the number of modeled relations, e.g., documents. Specifically, we evaluate (G1--G8) for G and (S1--S5) for S. A brief description of the use-cases and scenarios follows and the full parameter settings for all scenarios are reported in Appendix~\ref{appendix:testcases}.

\textbf{1. G.}
Models recursive folders and documents with concentric relations (viewer, commenter, writer), parent inheritance, and group- and user-side fan–out defined in the official OpenFGA guide. The series increases users (20→1000), groups (4→100), folders (8→200), and per-folder documents (3→30). Overlay parameters scale proportionally (agents 8→500, fixed 1 session/agent).

\textbf{2. S.}
Models workspaces and channels with roles (guest, admin), public/private visibility, and posting permissions. It exercises unions, implied relations (admin\,$\Rightarrow$\,writer\,$\Rightarrow$\,viewer), and workspace\,$\rightarrow$\,channel scoping. The series scales workspaces (2→120), channels/workspace (5→100), users (50→1200), and agents (5→300).

These use-cases represent contrasting structures: deep inheritance over content hierarchies (nested folders and documents) in G, and broad scope propagation across collaboration structures in S. They illustrate two structural poles of enterprise applications especially with document based AI architectures. We deem these use-cases representative of typical structures in the enterprise. Moreover, the scale of our scenarios exceeds what is used in the literature. We evaluate scenarios up to $1000$ users with $780k$ relations, while recent benchmarks such as Cedar benchmark evaluates $50$ users and ReBAC in data-spaces study reports a $120k$ relations~\cite{cutler2024cedar,fotiou2026relationship}. 

\paragraph{Methodology}
For each use-case we generate a paired dataset:
\textbf{1. Domain}: baseline tuples describing users, groups, workspaces, resources, and sharing structure. \textbf{2. Overlay}: the same domain tuples augmented with agents, sessions, scopes, and delegation chains. Both datasets share the same domain topology and sharing structure, so any performance difference is attributable to the additional overlay state and its runtime maintenance. We manually validate a small set of representative authorization checks for correctness and then scale both families through controlled parameter sweeps.

The benchmark goes beyond a read-only comparison. The \emph{Domain} configuration is executed as a \texttt{check}-only workload, whereas the \emph{Overlay} is executed as a mixed workload that interleaves checks and writes. In every scenario, we issue $1000$ operations; overlay runs use an $80/20$ check/write split. This construction reflects the target authorization setting, in which checks are issued continuously but the graph is also updated as the delegation state evolves.

A write operation models a minimal update by inserting a fixed three-tuple bundle: a \emph{delegatee} tuple that links an agent to a human principal, an \emph{actor} tuple that links the agent to a fresh session, and a \emph{holder} tuple that binds the session to an existing scope.

We implemented two Python generators that take structural parameters and emit two tuple files per case: \emph{Domain} and \emph{Overlay}. All random choices are seeded, with separate randomness governing domain construction and overlay augmentation. This yields reproducible workloads and supports parameter sweeps that vary one source of complexity at a time.  We vary three factor families:

\begin{enumerate}
	\item \textbf{Scale and Topology.}  
	Controls the structural size of each dataset.  
	For \textbf{G}, this includes \texttt{users}, \texttt{groups}, \texttt{folders}, and \texttt{docs-per-folder}, where folders form a shallow layered forest and group sizes include Gaussian variation.  
	For \textbf{S}, the parameters are \texttt{workspaces}, \texttt{channels-per-workspace}, and workspace user distributions.
	
	\item \textbf{Domain Fan–Out.}  
	Determines how broadly principals attach to objects.  
	In \textbf{G}, this is expressed via \texttt{group-viewer-ratio} and \texttt{doc-direct-viewer-ratio}. In \textbf{S}, the same notion governs role pressure on channels.  
	
	\item \textbf{Overlay State.}  
	Governs the size and complexity of the overlay, including \texttt{agents}, and \texttt{sessions-per-agent}.  Scopes follow a simple hierarchy. 
	Delegation chains are built from user→agent roots and extended 0–2 hops.  
\end{enumerate}

 In G, overlay runs mix agent-to-document, agent-to-folder, human-to-document, and human-to-folder checks. In S, overlay runs mix agent-writer and human-writer checks. For each case we measure $1000$ random checks on a 13th Gen Intel Core(TM) i7-1360P machine with 16.0 GB of memory, and report the following:
	 \textbf{1. Memory footprint:}  tuple count and average host memory usage during the run.
	 \textbf{2. Execution time:} mean, and median latency for all operations, together with separate summaries for \texttt{check} and \texttt{write} operations.

\paragraph{Results}
Our synthetic datasets simplify real deployments. Folder and channel topologies are generated from controlled templates, and delegation chains use bounded rules rather than full production histories. These choices can shift absolute latencies, but they do not change the comparative behavior between \emph{Domain} and \emph{Overlay} under identical seeds and scale settings.

\begin{table}[t]
	\centering
	\scriptsize
	\setlength{\tabcolsep}{3.2pt}
	\begin{tabular}{lccccc}
		\toprule
		Case & Tuples (\(10^3\)) D/O & Check Mean \(R\) & Check Med \(R\) & Write Med (\si{ms}) & Mem \(R\) \\
		\midrule
		G1 & \(0.12/0.15\)     & \num{1.10} & \num{1.10} & \num{5.15} & \num{0.99} \\
		G2 & \(0.21/0.26\)     & \num{1.20} & \num{1.22} & \num{4.81} & \num{1.00} \\
		G3 & \(0.80/0.87\)     & \num{1.15} & \num{1.13} & \num{5.20} & \num{1.01} \\
		G4 & \(1.30/1.41\)     & \num{1.22} & \num{1.11} & \num{5.68} & \num{1.02} \\
		G5 & \(6.23/6.48\)     & \num{1.55} & \num{1.16} & \num{7.77} & \num{1.02} \\
		G6 & \(16.74/17.10\)   & \num{2.09} & \num{1.27} & \num{6.95} & \num{1.16} \\
		G7 & \(131.0/131.5\)   & \num{2.16} & \num{1.32} & \num{6.19} & \num{1.07} \\
		G8 & \(\mathrm{766.6/786.3}\) & \num{2.2} & \num{1.30} & \num{9.42} & \num{1.2} \\
		\midrule
		S1 & \(0.09/0.12\)     & \num{1.01} & \num{1.00} & \num{4.50} & \num{1.00} \\
		S2 & \(0.26/0.32\)     & \num{1.05} & \num{1.05} & \num{4.49} & \num{1.00} \\
		S3 & \(\mathrm{12.4/16.6}\) & \num{1.80} & \num{1.17} & \num{6.94} & \num{1.00} \\
		S4 & \(24.88/33.5\)   & \num{2.10} & \num{1.21} & \num{5.18} & \num{0.95} \\
		S5 & \(373.3/386.7\)   & \num{2.16} & \num{1.22} & \num{5.23} & \num{0.98} \\
		\bottomrule
	\end{tabular}
	\caption{Per-case results. Tuple counts of Domain and overlay, \(R=\frac{\text{Overlay}}{\text{Domain}}\) for check mean, check median, and memory.
	}
	\label{tab:results-compact}
\end{table}

Table~\ref{tab:results-compact} reports check-to-check comparisons using
\emph{Overlay/Domain} ratios $R$ for mean and median latency, plus overlay write median, memory ratio, and the graph size. We report ratios to quantify the overlay’s relative impact, where values greater than 1 indicate an increase over the baseline and larger ratios correspond to stronger overhead.  As expected, the overlay adds graph size and memory footprint; however, memory remains bounded with ratios in \([\num{0.95}, \num{1.2}]\) across all cases, indicating no runaway amplification. The percentage of allowed checks (\textit{access granted}) remained around $40$\% across all the scenarios. More significantly, the overlay increases check latency, but in a manner that remains practical. 

In absolute terms (as plotted in Figure~\ref{fig:latency-absolute}), check-latency growth is modest at the median across cases. G medians range from \(\SI{1.97}{ms}\) (G1 Domain) to \(\SI{6.83}{ms}\) (G8 Overlay), while S medians remain even tighter: \(\SI{2.01}{ms}\) (S1 Domain) to \(\SI{4.47}{ms}\) (S5 Overlay). Domain benchmarks show medians of \(\SI{1.97}{ms}\text{--}\SI{5.3}{ms}\) for G and \(\SI{2.0}{ms}\text{--}\SI{3.7}{ms}\) for S, demonstrating that checks are already fast. Overlay medians extend this only moderately: G medians increase to \(\SI{2.17}{ms}\text{--}\SI{6.8}{ms}\) and S to \(\SI{2.0}{ms}\text{--}\SI{4.5}{ms}\), confirming that median performance remains practical for interactive systems.

The divergence between mean and median latencies, especially in larger scenarios, indicate heavy upper-tail effects. For both families, check-mean ratios range from \(\num{1.01}\times\) (S1) to \(\num{2.20}\times\) (G8), while check-medians remain close to unity. This illustrates that there are slower \textit{check} queries that arise especially in bigger models.


In summary, even for deep hierarchies median latency remain under $7ms$ (G8), confirming real-time viability. Similarly, write operations remain efficient with medians in the range \(\SI{4.49}{ms}\text{--}\SI{9.42}{ms}\) across the reported suite. 
Without engine-level tuning, caching, or indexing beyond OpenFGA's defaults, these measurements represent a conservative lower bound. 
Standard optimizations, e.g., by caching usersets, normalizing delegation chains, sharding the graph, would further compress latency. Thus, the overlay approach is shown to be \textit{feasible and practical} within target systems.

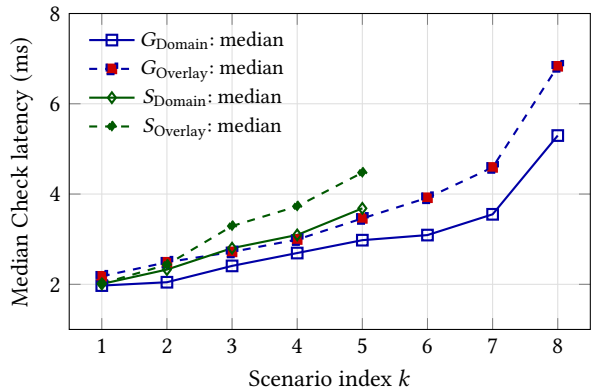
\begin{figure}[t]
	\centering
	\begin{tikzpicture}
		\begin{axis}[
			width=\linewidth,
			height=0.68\linewidth,
			xlabel={Scenario index \(k\)},
			ylabel={Median Check latency \((\si{ms})\)},
			xmin=0.5, xmax=8.5,
			ymin=1, ymax=8,
			xtick={1,2,3,4,5,6,7,8},
			xticklabel style={anchor=north},
			grid=both,
			major grid style={gray!25},
			minor grid style={gray!15},
			ymajorgrids=true,
			legend columns=1,
		legend style={at={(0.02,0.98)},anchor=north west,draw=none,fill=none,font=\small},
			]
			
			
			\addplot+[mark=square, thick, solid, blue!65!black] coordinates {
				(1,1.972) (2,2.046) (3,2.408) (4,2.693)
				(5,2.978) (6,3.091) (7,3.551) (8,5.296)
			};
			\addlegendentry{$G_{\mathrm{Domain}}$: median}
			
			
			\addplot+[mark=square*, thick, dashed, blue!65!black] coordinates {
				(1,2.178) (2,2.489) (3,2.718) (4,2.992)
				(5,3.460) (6,3.918) (7,4.59) (8,6.83)
			};
			\addlegendentry{$G_{\mathrm{Overlay}}$: median}
			
%
			\addplot+[mark=diamond, thick, solid, green!35!black] coordinates {
				(1,2.008) (2,2.327) (3,2.801) (4,3.094) (5,3.683)
			};
			\addlegendentry{$S_{\mathrm{Domain}}$: median}
			
			
			\addplot+[mark=diamond*, thick, dashed, green!35!black] coordinates {
				(1,2.017) (2,2.435) (3,3.291) (4,3.728) (5,4.474)
			};
			\addlegendentry{$S_{\mathrm{Overlay}}$: median}
			
		\end{axis}
	\end{tikzpicture}
	\caption{Absolute median check-latency values for Domain and Overlay across G and S scenarios, showing practical growth.}
	\label{fig:latency-absolute}
\end{figure}

	\section{Related Work}\label{sec:rel}
\textbf{Governance for Agentic AI.}
Recent work on AI governance spans system design, identity, and runtime control. Zhang et al. apply classical security principles such as defense-in-depth to agents via a conceptual ``AgentSandbox'' centered on policy enforcement and data minimization~\cite{Zhang}. Syros et al.\ present a centralized architecture combining user-centric governance, cryptographic tokens, and a provider registry to mediate agent communication under user policies~\cite{Syros}. Huang et al.\ outline a zero-trust framework based on Verifiable Credentials as a complement to traditional IAM \cite{Huang}. Wang et al.\ move closer to our setting with a runtime governance stack that monitors delegation provenance for auditing \cite{Wang}. Across these efforts, however, authorization is typically expressed through RBAC/ABAC-style controls or token-based mechanisms rather than through delegation chains and scope envelopes treated as first-class predicates inside access rules. By contrast, we focus on \emph{compositional} authorization semantics that embed delegation relations directly into the policy model for evaluation.

Palumbo et al.~\cite{palumbo2026policy} propose a compiler  that instruments agent implementations with data rules. Their architecture emphasizes static rule compilation, but does not address delegation or scoping. Our work instead provides a compositional operator for overlaying these primitives onto authorization schemas. 
Closely related, Potti proposes Intent-Based Access Control (IBAC), where an LLM maps user intent, expressed in a prompt, to tool permissions encoded as OpenFGA tuples \cite{potti2024ibac}.
 Although IBAC uses similar tools, our work differs in purpose.
 While IBAC focuses on inferring what an agent may do from a prompt, our work specifies the \emph{relational logic} by which authority is delegated, bounded, and inherited across agents.

Tomasev et al.\ propose a framework for intelligent delegation that formalizes how agents decide when and to whom to delegate in agent systems~\cite{tomavsev2026intelligent}. Their focus is complementary to ours: they study delegation as a decision problem, whereas we study how a delegation, once made, should be represented as an enforceable authorization primitive and tracked through recursive chains.

\textbf{Agentic Protocols and Standards.}
A parallel branch studies how existing standards map to agentic delegation. The OpenID community has analyzed where OAuth can be extended and where limitations arise for agents~\cite{openid}. Proposals such as South et al.\ \cite{authenicatedDEL} illustrate such an extension. Decentralized-identity approaches (e.g., DID/VC-based systems~\cite{s24072215,damfsd}) articulate device-centric delegation that could be adapted for agents. Here, ``delegation'' is realized as secure credential issuance and presentation; in contrast, we treat delegation as a \emph{first-class relation} in authorization semantics, evaluated in concert with scope constraints. Integrating our model with decentralized identity infrastructure is promising future work.

\textbf{Authorization Models and Policy Composition.}
ReBAC systems e.g., Zanzibar \cite{Pang2019} and implementations like OpenFGA represent policies as typed relations and reduce checks to reachability~\cite{cheng2012relationship, giunchiglia2008relbac}. While these systems support expressive relation definitions, they generally do not prescribe \emph{model composition} mechanisms for injecting new primitives into  existing models. Policy-combination and algebraic operators (e.g., \cite{BonattiSamarati2002}) address decision aggregation and conflict resolution across multiple policies, whereas our contribution \emph{fuses} primitives (delegation chains, scope  envelopes) into the graph semantics so they are evaluated natively within checks.

\textbf{Authorization Logics and Trust Management.}
Authorization logics and trust-management systems, such as~\cite{RFC2693, RivestLampson1996,AppelNAL2014}, model principals, credentials, and delegation as logical statements; authorization reduces to proof search that a requester satisfies a capability. These frameworks provide strong foundations for \emph{delegation}, and \emph{attenuation via constraints}, typically consuming external credentials (certificates) as inputs to derivations. Our approach is complementary: we adopt a ReBAC relation model and introduce a \emph{compositional overlay} that injects delegation chains and scope envelopes directly into the authorization graph. The two perspectives can interoperate: logical proofs can materialize overlay tuples or satisfy guard conditions, while the overlay provides a scalable substrate for applying such evidence across checks and listings.

	\section{Conclusion}\label{sec:conclusion}
Agentic AI introduces a new operational model in which autonomous agents can act, reason, delegate, and collaborate with minimal human supervision. 
Such behavior challenges long-standing assumptions in IAM, where delegation is typically modeled as a static, token-mediated act. 
Modern agent ecosystems, however, require delegation and scoping to function as \emph{dynamic governance primitives} that support continuous enforcement and auditability. 

This paper presented an authorization framework that elevates delegation, scope, and contextual constraints to first-class constructs. We developed a taxonomy of delegation suitable for agents, introduced the notion of authorization envelopes, and formalized a
model capturing users, agents, sessions, and scopes as relationships. At the core of our contribution is a \emph{compositional overlay operator} that injects agentic semantics
into ReBAC policies. Grounded in non-deleting typed graph rewriting, the operator is conservative for domain principals while adding agent permissions only through a
human-rooted delegation chain and a compatible  envelope.
We proved this as an agent-authorization soundness condition: an agent
may obtain a lifted permission only when there exists an authorized
human principal, a valid delegation path from that principal to the
agent, and a valid scope/session witness for the requested object.

We operationalized these ideas through \emph{ACE}, an architectural
blueprint designed to integrate heterogeneous sources, context, and
dynamic delegation state into an authorization layer for agents.
A multi-agent coding assistant illustrated how enterprise policies can
be extended with delegation and scoping semantics while preserving the
underlying human-facing permission structure.

Our evaluation combined formal reasoning with empirical benchmarks on
large-scale models. The agentic overlay increases graph size and check
latency as expected, yet median check times remain under \(7\,\mathrm{ms}\)
even without specialized optimization.

Looking forward, several directions remain open:
interoperability with other authorization models, e.g., ABAC;
rigorous multi-policy composition; formal treatment of
scope attenuation; and authorization graph engines
tailored for large, dynamic agent populations. As agentic AI becomes
increasingly ubiquitous, our work provides a principled and extensible
foundation for building secure, accountable, and context-aware
authorization mechanisms capable of governing autonomous software
actors at scale.
	
	\bibliographystyle{ACM-Reference-Format}
	\bibliography{./sources.bib}
	
	\appendix
	
	\section{Open Science} 
All artifacts are included in the submitted supplementary material and are accessible to reviewers through the anonymous artifact URL:  \url{https://github.com/Amjad-Ibrahim-Huawei/compositional-paper}. This appendix enumerates all artifacts necessary to evaluate and reproduce the paper's core contributions.  

\begin{enumerate}

	\item{\textbf{Documentation}}
	
	\begin{itemize}
		\item {\texttt{openfga/general/6.gdrive/agent-ai/README.md}}
		\\ Official documentation of the GDrive use-case, model decisions, and expected performance characteristics.
		
		\item {\texttt{openfga/general/slack/README.md}}
		\\ Official documentation of the Slack benchmark scenario, authorization semantics, and experimental parameters.
		
		\item {\texttt{openfga/general/Evaluation\_Commands}}
		\\ Step-by-step execution commands for reproducing the benchmark.
	\end{itemize}

	\item{\textbf{Benchmarking Infrastructure}}
	
	\begin{itemize}
		\item {\texttt{openfga/general/benchmark.py}}
		\\ Main benchmark driver implementing the evaluation protocol. Executes checks, collects performance metrics, and logs execution traces.  
		
		\item {\texttt{openfga/general/setup\_and\_load.sh}}
		\\ Orchestration script for store initialization, model loading, and tuple population. Handles environment configuration and data ingestion required to prepare benchmarks.
		
		\item {\texttt{openfga/general/setup\_store.sh}}
		\\ Store creation and model schema initialization. Deploys authorization models to OpenFGA instances.
		
		\item {\texttt{openfga/general/delete\_store.sh}}
		\\ Cleanup utility between benchmark runs.
	\end{itemize}

	\item{\textbf{Data Generation and Tuple Population}}
	
	\begin{itemize}
		\item {\texttt{openfga/general/openfga\_tuple\_dataset\_generator.py}}
		\\ Generates synthetic relation tuple datasets for the GDrive scenario. Implements domain-specific rules for creating user-resource relationships at scale.
		
		\item {\texttt{openfga/general/openfga\_tuple\_slack\_generator.py}}
		\\ Generates synthetic relation tuple datasets for the Slack scenario. Populates workspace, channel, and user relationships according to Slack's authorization model.
		
		\item {\texttt{openfga/general/rebuild\_analysis\_from\_raw.py}}
		\\ Post-processing utility that transforms raw benchmark output into analysis-ready formats. Aggregates metrics and computes summary statistics.
	\end{itemize}
	
	\item {\textbf{Authorization Models}}
	
	\begin{itemize}
		\item {\texttt{openfga/general/6.gdrive/gdrive-domain.fga}}
		\\ Core G model defining relationships (owners, editors, viewers) and permission logic for document access control.
		
		\item {\texttt{openfga/general/6.gdrive/agent-ai/}}
		\\ G Overlay model variant.

		\item {\texttt{openfga/general/slack/model.fga}}
		\\ Core S model defining workspace and channel permission semantics.
		
		\item {\texttt{openfga/general/slack/agent-ai/}}
		\\Overlay model for Slack scenarios with AI integration.
	\end{itemize}

	\item { \textbf{Relation Tuple Datasets}}
	
	\begin{itemize}
		\item {\texttt{openfga/general/6.gdrive/agent-ai/generated/}}
		\\ Synthetic relation tuples for GDrive scenario. Contains domain and overlay files (G1--G7). (G8 files are around 80MiB so were excluded due to size limit; but they can be reproduced using the scripts as shown in the commands.)
		
		\item {\texttt{openfga/general/slack/agent-ai/generated/}}
		\\ Synthetic relation tuples for Slack scenario. Contains domain and domain files (S1--S4). (S5 files are around 50 MiB so were excluded from the repository due to size limit; but they can be reproduced using the scripts as shown in the commands.)
	\end{itemize}
	
	\item {\textbf{Experimental Results and Analysis}}
	
	\begin{itemize}
		\item {\texttt{openfga/general/results/analysis/}}
		\\ Aggregated analysis outputs, summary statistics, and processed metrics derived from raw benchmark runs.
		
		\item {\texttt{openfga/general/results/model\_7/}}
		\\ Benchmark CSV outputs for the baseline authorization model across all datasets (check-only queries on G1--G8).
		
		\item {\texttt{openfga/general/results/model\_8/}}
		\\ Benchmark CSV outputs for an optimized variant across all datasets (mixed query types on G1--G8).
		
		\item {\texttt{openfga/general/results/model\_slack\_domain/}}
		\\ Benchmark CSV outputs for the Slack domain model (check-only queries on S1--S5).
		
		\item {\texttt{openfga/general/results/model\_slack\_overlay/}}
		\\ Benchmark CSV outputs for the Slack overlay variant (mixed query types on S1--S5).
	\end{itemize}

\end{enumerate}

\section{Generative AI Usage}
We used OpenAI's ChatGPT (GPT-5.4, Plus plan) as an assistant during the preparation of this manuscript. Specifically, LLMs were used for editorial purposes (language polishing, clarification of phrasing, and restructuring of paragraphs), for suggesting \LaTeX{} snippets (e.g., tables, and plotting code), and for drafting Python scaffolding to generate synthetic OpenFGA tuples and parameterized test datasets. LLMs were used for editorial purposes in this manuscript, and all outputs were inspected by the authors to ensure accuracy and originality. All technical ideas, and experimental designs are our own; any code or data-generation logic initially drafted with the help of ChatGPT was subsequently reviewed, simplified, and re-implemented or directly validated by the authors, and all experiments reported in the paper can be reproduced from the code and parameters we explicitly provide. We did not use LLMs as a source of prior work or citations and relied on our own literature review for related work. 

\section*{Expansion of the Agentic Overlay}\label{sec:rewriteAppendix}
Table~\ref{tab:rewrite-rules} provides an implementation-oriented expansion of the overlay macros used by the operator. Each row summarizes the corresponding schema fragment and its intended effect.
\begin{table*}[!t]
	\centering
	\caption{Implementation-Oriented Expansion of the Agentic Overlay Macro}
	\renewcommand{\arraystretch}{1.12}
	\setlength{\tabcolsep}{5pt}
	\begin{tabular}{p{0.5cm}|p{1.5cm}|p{3cm}|p{10cm}}
		\toprule
		\textbf{Rule} & \textbf{Scope} & \textbf{Intent} & \textbf{LHS / RHS / Effects} \\
		\midrule
		
		\textbf{R0} & Overlay Bootstrap (once per model) & Ensures global agentic scaffolding and entrypoints. &
		\textbf{LHS:} model lacks \texttt{agent}, \texttt{session}, \texttt{scope}. 
		\newline \textbf{RHS:}
		\begin{minipage}[t]{\linewidth}\vspace{2pt}
			\begin{lstlisting}
				type agent, type session, type scope;
				scope.holder: [session] (with optional temporal guard);
				session.actor: [agent], session.as_agent := actor;
				scope.sessions := holder or sessions from parent;
				entrypoints:
				scope.agents_direct := as_agent from holder,
				scope.agents := agents_direct or agents from parent.
			\end{lstlisting}
		\end{minipage}
		\newline \textbf{Effect:} valid path from agent subject to user relation (Zanzibar/OpenFGA). \\\midrule
		
		\textbf{R1} & Scope Linkage (per type) & Attach domain type $T$ to scope; derive agents active in envelope. &
		\textbf{LHS:} type $T$ without \texttt{in\_scope}.
		\newline \textbf{RHS:}
		\begin{minipage}[t]{\linewidth}\vspace{2pt}
			\begin{lstlisting}
				in_scope: [scope]
				ags_in_scope: agents from in_scope
			\end{lstlisting}
		\end{minipage}
		\newline \textbf{Effect:} determines which agents are active for $T$ under context. \\\midrule
		
		\textbf{R2} & Delegation Chain Lift (per relation family) & Lift human permissions (viewer/editor/owner) to their delegates recursively. &
		\textbf{LHS:} $T$ with $\{\mu_T(\text{viewer}),\mu_T(\text{owner}),\mu_T(\text{editor}),\mu_T(\text{parent})\}$.
		\newline \textbf{RHS:}
		\begin{minipage}[t]{\linewidth}\vspace{2pt}
			\begin{lstlisting}
				chain_agents_for_rel_T =
				can_execute_on_my_behalf from {mu.viewer}
				or can_execute_on_my_behalf from {mu.owner}
				or can_execute_on_my_behalf from {mu.editor}
				or chain_agents_for_rel_T from {mu.parent}
			\end{lstlisting}
		\end{minipage}
		\newline \textbf{Effect:} collects all delegated agents for relation \texttt{rel}; conditions evaluated per check. \\\midrule
		
		\textbf{R3} & Intersection Gate (per relation family) & Require scope-membership AND delegation-chain. &
		\textbf{LHS:} $\texttt{ags\_in\_scope\_T and chain\_agents\_for\_*\_T}$
		\newline\textbf{RHS:}
		\begin{minipage}[t]{\linewidth}\vspace{2pt}
			\begin{lstlisting}
				delegated_agent_rel_T = ags_in_scope_T and chain_agents_for_rel_T
			\end{lstlisting}
		\end{minipage}
		\newline \textbf{Effect:} implements intersection $A \cap B$. \\\midrule
		
		\textbf{R4} & Agent Injection (per relation family) & Union delegated agents into existing permission. &
		\textbf{LHS:} $\texttt{existing can\_view or can\_edit on (T)}$
		\newline
		\textbf{RHS:}
		\begin{minipage}[t]{\linewidth}\vspace{2pt}
			\begin{lstlisting}
				can_rel = (existing) or delegated_agent_rel_T
			\end{lstlisting}
		\end{minipage}
		\newline \textbf{Effect:} agents share the same effective permission as humans. \\\midrule
		
		\textbf{R5} & Delegatee Condition Enrichment (global) & Add temporal/attribute guards on \texttt{delegatee} edges. &
		\textbf{LHS:} $\texttt{user.delegatee, agent.delegatee}$
		\newline
		\textbf{RHS:} \texttt{user.delegatee}, \texttt{agent.delegatee} allow \texttt{temporal\_delegation}, \texttt{conditional\_delegation}.
		\newline \textbf{Effect:} dynamic, context-bound delegation evaluation. \\\midrule
		
		\textbf{R6} & Temporal/ Scoping Guards (global) & Place temporal/audience restrictions on scope entries. &
		\textbf{LHS:} $\texttt{scope.holder}$
		\newline
		\textbf{RHS:} \texttt{scope.holder: [session with temporal\_delegation]}.
		\newline \textbf{Effect:} time-checked once via scope holder. \\
		
		\bottomrule
	\end{tabular}
	\label{tab:rewrite-rules}
\end{table*}

\section*{Evaluation Test Cases Table}\label{appendix:testcases}
Table~\ref{tab:params} show the detailed parameters used for each case shown in the evaluation. 
\begin{table*}[t]
	\centering
	\small
	\begin{tabular}{@{}llrrrrrrr@{}}
		\toprule
		Case & Family & Users & Groups & Folders & Docs/F & Agents & Sess/Ag & Notes \\
		\midrule
		G1 & Drive & 20  & 4  & 8   & 3  & 8   & 1 & GVR=0.5, DVR=0.15, SO=0.25\\
		G2 & Drive & 20  & 8  & 12  & 3  & 12  & 1 & GVR=0.5, DVR=0.15, SO=0.25\\
		G3 & Drive & 60  & 6  & 12  & 4  & 20  & 1 & GVR=0.5, DVR=0.4,  SO=0.25\\
		G4 & Drive & 100 & 10 & 20  & 8  & 33  & 1 & GVR=0.5, DVR=0.1,  SO=0.25\\
		G5 & Drive & 200 & 20 & 40  & 12 & 70  & 1 & GVR=0.5, DVR=0.1,  SO=0.25\\
		G6 & Drive & 300 & 30 & 60  & 16 & 100 & 1 & GVR=0.5, DVR=0.1,  SO=0.25\\
		G7 & Drive & 500 & 50 & 100 & 20 & 150 & 1 & GVR=0.5, DVR=0.25, SO=0.5\\
		G8 & Drive & 1000 & 100 & 200 & 30 & 500 & 1 & GVR=0.5, DVR=0.25, SO=0.5\\
		\midrule
		S1 & Slack & 50  & -- & --  & -- & 5   & 1 & WS=2, Ch/W=5, Writers/Ch=2 \\
		S2 & Slack & 120 & -- & --  & -- & 8   & 1 & WS=4, Ch/W=10, Writers/Ch=2 \\
		S3 & Slack & 400 & -- & --  & -- & 50  & 2 & WS=40, Ch/W=100, Writers/Ch=2 \\ 
		S4 & Slack & 800 & -- & --  & -- & 150 & 2 & WS=80, Ch/W=100, Writers/Ch=2 \\
		S5 & Slack & 1200& -- & --  & -- & 300 & 2 & WS=120, Ch/W=100, Writers/Ch=30, Temp=0.1 \\
		\bottomrule
	\end{tabular}
	\caption{Details of Evaluation test cases with Key parameters per case. (GVR=\texttt{group\_viewer\_ratio}, DVR=\texttt{doc\_direct\_viewer\_ratio}, SO=\texttt{session\_only\_fraction}, WS=workspaces, Ch/W=channels per workspace.)}
	\label{tab:params}
\end{table*}
	
\end{document}